\documentclass[letterpaper, 10 pt, conference]{ieeeconf}  % Comment this line out if you need a4paper
\IEEEoverridecommandlockouts                              % This command is only needed if
\usepackage{color}
\usepackage{times}
\usepackage{epsfig}
\usepackage{graphicx}
\usepackage{tabularx,ragged2e,booktabs}
\usepackage{amsmath}
\usepackage{amssymb}
\usepackage{algpseudocode}
\usepackage{algorithmicx}
\usepackage[ruled,vlined,linesnumbered,noend]{algorithm2e}
\usepackage{graphicx} %% This one is the package for using \includegraphics
\usepackage{diagbox}
\usepackage{graphbox}
\usepackage{tikz}
\usepackage{booktabs}
\usepackage{multirow}
\usepackage{makecell}
\usepackage{authblk}
\usepackage{hyphenat} % for preventing hyphenation, using environment \nohyphens{(text block)}
\usepackage{pgf}
\usetikzlibrary{arrows,automata}

\newtheorem{lemma}{Lemma}

\DeclareMathOperator*{\argmax}{arg\,max}

\graphicspath{{figures/}}

\usepackage{hyperref}

\title{\Large \bf HGP-RL: Distributed Hierarchical Gaussian Processes for Wi-Fi-based Relative Localization in Multi-Robot Systems}

\author{Ehsan Latif \and Ramviyas Parasuraman
\thanks{The authors are with the School of Computing, University of Georgia, USA.  Corresponding Author Email: {\tt\small \{ramviyas\}@uga.edu}.}  
\thanks{This work is supported by the U.S. Army Research Laboratory under the DCIST Cooperative Agreement Number W911NF-17-2-0181.}
}

\begin{document}

\maketitle
\thispagestyle{empty}
\pagestyle{empty}

\begin{abstract}
Relative localization is crucial for multi-robot systems to perform cooperative tasks, especially in GPS-denied environments. Current techniques for multi-robot relative localization rely on expensive or short-range sensors such as cameras and LIDARs. As a result, these algorithms face challenges such as high computational complexity (e.g., map merging), dependencies on well-structured environments, etc. To remedy this gap, we propose a new distributed approach to perform relative localization (RL) using a common Access Point (AP). To achieve this efficiently, we propose a novel Hierarchical Gaussian Processes (HGP) mapping of the Radio Signal Strength Indicator (RSSI) values from a Wi-Fi AP to which the robots are connected. Each robot performs hierarchical inference using the HGP map to locate the AP in its reference frame, and the robots obtain relative locations of the neighboring robots leveraging AP-oriented algebraic transformations.
The approach readily applies to resource-constrained devices and relies only on the ubiquitously-available WiFi RSSI measurement. We extensively validate the performance of the proposed HGR-PL in Robotarium simulations against several state-of-the-art methods. The results indicate superior performance of HGP-RL regarding localization accuracy, computation, and communication overheads. Finally, we showcase the utility of HGP-RL through a multi-robot cooperative experiment to achieve a rendezvous task in a team of three mobile robots. 
%Our approach, Gaussian Processes-based Relative Localization (HGP-RL), combines two pillars. First, the robots locate the AP w.r.t. their local reference frames using novel hierarchical inferencing that significantly reduces computational complexity. Secondly, the robots obtain relative positions of neighbor robots with an AP-oriented vector transformation. 
%even when robots are not in direct line of sight or have different initial orientations. 
%which is ubiquitously available in any wireless-capable device.  
%These limitations make them impractical for resource-constrained robots. 
%Since their introduction and widespread use, wireless local area networks have drawn more attention to Wi-Fi-based localization systems. 

\end{abstract}


\IEEEpeerreviewmaketitle

\section{Introduction}
\label{sec:intro}
Multi-robot systems (MRS) have recently drawn significant attention for various use cases, including logistics, surveillance, and rescue. In GPS-denied environments or applications where the privacy of absolute (global) location must be protected, using the robot's relative position to other robots or environment markers is essential as the robots need to cooperate, share data, and complete jobs effectively \cite{latif2023}. Here, relative localization is key in cooperative multi-robot tasks such as rendezvous, formation control, coverage, planning, etc., as well as executing swarm-level behaviors.
%While the progress in simultaneous localization and mapping (SLAM) techniques applied to a mobile robot has reached a significant research maturity \cite{cadena2016past}, relative localization in an MRS is still in the exploratory research stage, given the exploding complexity of localizing multiple robots simultaneously. 
%For instance, the state-of-the-art Kimera-Multi \cite{tian2022kimera} for multi-robot SLAM use RBG-D camera data to obtain a metric-semantic 3-D mesh model of the environment with excellent accuracy. 

While the progress in simultaneous localization and mapping (SLAM) techniques has reached a significant research maturity \cite{cadena2016past,tian2022kimera}, they rely on computationally expensive sensors such as RGB-D cameras and LIDARs
%the necessity to have neighboring robots within close range of each other for the sensor data to detect/overlap and perform relative localization. 
Moreover, the localization mechanism derived from these sensors suffers from the necessity to acquire loop closure through environmental mapping overlap, creating the co-dependence of localization and mapping objectives. 
The research challenges are even more prominent for resource-constrained robots, which have limited computation and sensing capabilities. 
%creating spatio-temporal dependencies on the robots' movements. 
%these techniques rely on cameras and lidars and the computationally-intensive 
In fact, relative localization can be sufficient (without the need for SLAM-based feature-matching and map merging) to perform major cooperative multi-robot tasks like rendezvous, formation control, etc.  \cite{parasuraman2019consensus,chen2022survey}.
Therefore, we focus on alternative sensors such as UWB and Wi-Fi for the relative localization of robots without using environmental maps to avoid the drawbacks of map-based localization.
%in While they may only have limited sensory input in many real-world situations; as a result, it is crucial to perform relative localization using the data at hand.

\begin{figure}[t]
\centering
%\vspace{-7mm}
 \includegraphics[width=0.99\linewidth]{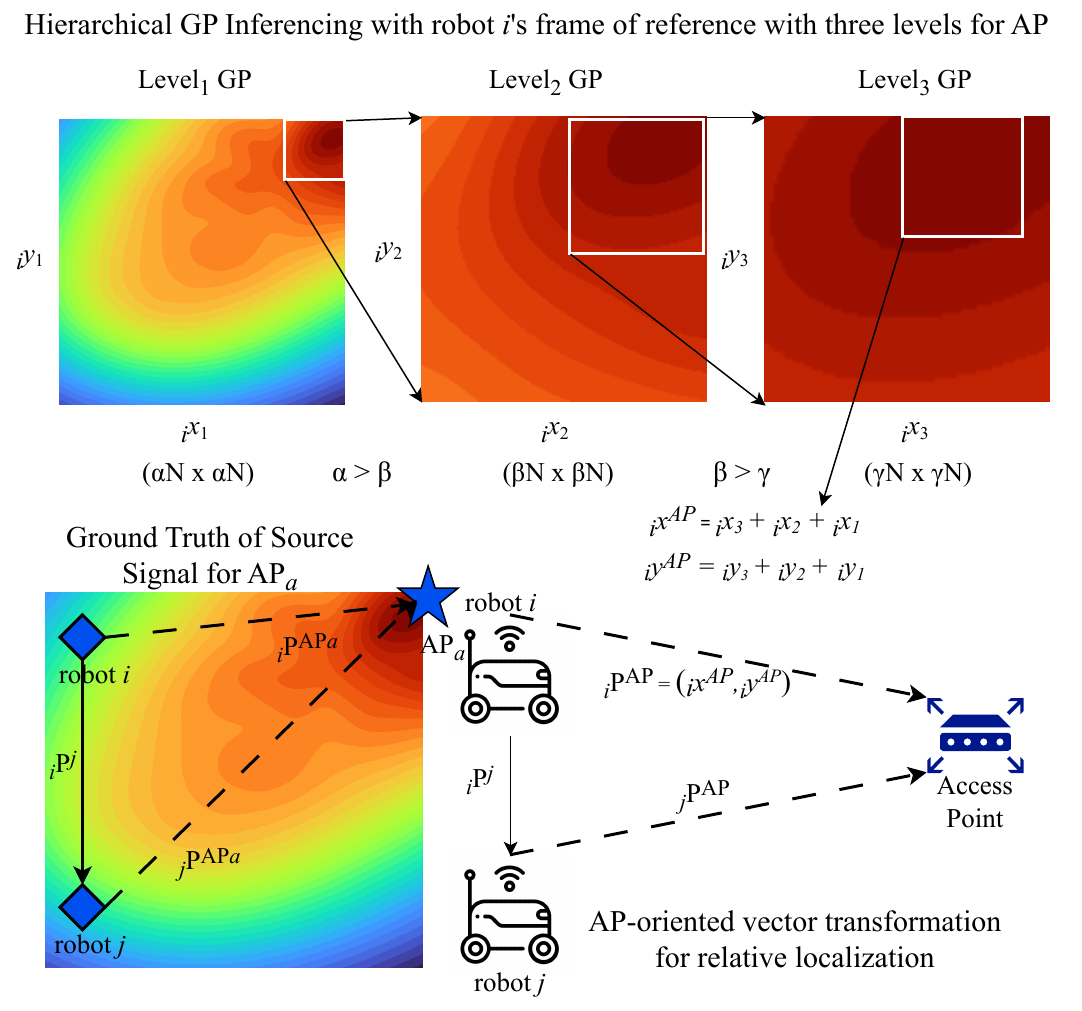}
 \caption{An overview of HGP-RL with three levels of Gaussian Processes and relative localization using Access Point position. $\alpha,\beta$ and $\gamma$ are resolution parameters for $N\times N$ grid space, $_i\mathbf{p}^{AP}$ and $_j\mathbf{p}^{AP}$ are the position vectors from the robot's position to the predicted AP position, and $_i\mathbf{p}^j$ is the relative position vector from robot $i$ to $j$.}
 \label{overview}
 \vspace{-4mm}
\end{figure}

Due to their wide accessibility, Wi-Fi signals present a promising data supply for relative localization tasks. The distance between the robot and the Wi-Fi Access Point (AP) can be estimated using the Received Signal Strength Indicator (RSSI) of Wi-Fi signals \cite{parasuraman2014multi}. However, variables, including multipath fading, shadowing, and background noise, frequently impact RSSI-based localization \cite{pandey2022empirical}. This drives the need for reliable and effective learning algorithms that can accurately obtain relative localization by utilizing the given RSSI data in the face of these difficulties \cite{latif2022dgorl,tardioli2010enforcing,latif2022multi}.
Gaussian Processes (GP) based methods have been explored in the literature \cite{fink2010online,xue2019locate} to model the Wi-Fi signal propagation using RSSI measurements, but they come at a high compute cost. While the training complexity of GPR has been well-investigated in the literature, its inference complexity is a key challenge in using GP-based methods for dominant source search in a large and dense environment, where the prediction (inference) time complexity explodes when the search resolution is very high. 

%TODO% Privacy-preserving - Robots cannot share their location directly. Instead, they can share information about something (like the location of AP) that other robots can trust.

%For tackling relative localization issues, online robotic learning techniques have been used. 
%Bayesian optimization (BO) \cite{zhu2020indoor}, and deep learning \cite{berkenkamp2021bayesian} have recently received attention as adaptive approaches to learning and enhancing localization performance in dynamic and unpredictable situations. Even though these methods have been effective, they frequently necessitate a large quantity of data to converge to an ideal solution, which may not be practicable in real-world situations. For their capacity to model the uncertainty in RSSI observations, probabilistic approaches like GPR have drawn attention \cite{fung2019coordinating}. A more precise portrayal of the underlying uncertainty is made possible by GPR, which models the RSSI-distance connection as a stochastic process. GPR-based techniques, however, can be computationally costly, particularly in high-resolution search areas.

Therefore, we propose a novel distributed algorithm that overcomes the drawbacks of existing approaches to AP location prediction (source search) and relative localization. 
% Our approach uses hierarchical inferencing on the Gaussian Processes regression (GPR \cite{quinonero2005unifying}) map of RSSI to accurately detect the location of the Wi-Fi AP in its local frame of reference with high computational and real-time efficiency. Then, we apply an AP-oriented vector transformation process that allows robots to accurately localize against each other by transforming other robots' coordinates into their own reference frame.
Fig.~\ref{overview} delineates the high-level overview of the two-stage process behind the proposed distributed Hierarchical Gaussian Process for Relative Localization (HGP-RL).
The proposed approach deals with the issue of GPR inference complexity by incorporating hierarchical inferencing, with efficient and distributed processing at each robot for varying dimension spaces and the number of robots in the system.

%illustrating the process behind the hierarchical inferencing for efficient AP localization and using this outcome to estimate the relative locations of other robots.
%\cite{yin2019mrs}
% Using hierarchical inferencing to take advantage of the GPR model's structure and smoothness, our method accurately predicts the location of the Wi-Fi source while avoiding local optimum conditions. The suggested vector transformation method ensures that relative localization is reliable even when the robots do not face each other directly or have different initial orientations.
%The proposed method satisfies the following important characteristics required for a multi-robot algorithm: scalable, range, practical, 
%The proposed method offers an accurate and efficient learning-based relative localization solution by employing GPR on RSSI data. 
The key contributions of this paper are two-fold:
\begin{itemize}
    \item Leveraging on the fact that the RSSI map of an AP is an unimodal distribution, we propose an efficient approach to infer the AP (source) location by hierarchically searching the RSSI map going from a sparse resolution GPR map to a denser resolution GPR map significantly reducing the search complexity from $\mathcal{O}(N^d)$ for a d-dimensional map with N grid points in each dimension to $\mathcal{O}(k (\lambda N)^d)$, where $\lambda$ is a sparsification of resolution at each level of GPR with $\lambda N << N$, and $k$ being the number of levels in the hierarchical GPR. 
    % This allows efficient and accurate AP position prediction relative to the robot, satisfying the computing resource constraints for each robot. This computation reduction is crucial to realize a truly distributed implementation of GPR-based algorithms for source-seeking and localization applications (e.g., on small UAVs with limited onboard computation).
    The novelty of this approach can be extended to other target search and source localization applications for both single and multi-robot scenarios.
    \item Leveraging the fact that all robots connect to and locate the same physical AP in their internal reference frames, we propose an AP-oriented relative localization mechanism leveraging conventional algebraic techniques fusing odometry and IMU data to locate neighbor robots from each robot's frame of reference. The novelty of this approach lies in the scalability and the capability to perform relative localization using ubiquitous sensors on each robot (e.g., WiFi and IMU), allowing implementations on robots with resource and SWaP-constraints.
    %\item We theoretically analyze the approach to evaluate the efficiency and accuracy of HGP-RL.
\end{itemize}

We theoretically analyze the accuracy and efficiency and extensively validate the HGP-RL's performance in the Robotarium-based simulations \cite{wilson2021robotarium}, compared against relevant state-of-the-art approaches in each stage (AP localization and relative localization). 
%: state-of-the-art Terrain Relative Navigation (TRN) \cite{wiktor2020ICRA}, Distributed Graph Optimization for Relative Localization (DGORL) \cite{latif2022dgorl}, and Modified Error Gaussian Process Regression (MEGPR) \cite{xue2019locate}.
We also demonstrate the practicality of HGP-RL using real robot experiments implemented in the ROS framework. These experiments demonstrate the applicability of HGP-RL's relative localization outcome to a multi-robot consensus algorithm, where all robots use HGP-RL to achieve the rendezvous objective.
Finally, we release the source codes (both Robotarium and ROS packages) in Github\footnote{\url{https://github.com/herolab-uga/hgprl}} for use and further development by the robotics community.

\section{Related Work}
\label{sec:relatedwork}
%TODO - first - robot/AP localization using GP mapping. Second, multi-robot relative localization approaches.
%\paragraph*{Wireless Signal Mapping for Robot Localization}
In the literature, relative localization has received substantial study, and several multi-robot system concepts have been put forth \cite{prorok2012low,nguyen2021flexible,wanasinghe2015relative}. 
%Relative localization and its significance in collaborative robotic systems were thoroughly covered by Wanasinghe et al. \cite{wanasinghe2015relative} and Rone and Ben-Tzvi \cite{rone2013mapping}. 
%These methods, however, frequently lacked the flexibility to adapt to changing settings. 
%A Jacobian-free strategy for multi-robot relative localization was introduced in \cite{wanasinghe2014jacobian}, which increased processing efficiency but still had difficulties adjusting to changing settings. 
A low-cost embedded system for relative localization in robotic swarms was proposed in \cite{faigl2013low} to lower hardware costs. Still, the localization performance remained sensitive to environmental changes. 
%Wanasinghe et al. \cite{wanasinghe2014distributed} investigated distributed collaborative localization, and 
Recently, Wiktor and Rock \cite{wiktor2020ICRA} presented a Bayesian optimization-based approach for collaborative multi-robot localization in natural terrain, but the method comes with a high complexity of information fusion and computational costs.
In our previous study \cite{latif2022dgorl}, we proposed a graph-theoretic approach to relative localization. While it addressed some scalability issues of MRS, it still encounters high computational complexity and requires range (or RSSI) sensing between each and every robot combination, which may not be practical in highly resource-constrained systems. Therefore, a system relying on a single environmental anchor like the Access Point of a wireless connection could be more practical.

%\paragraph*{Multi-Robot Relative Localization}
Approaches to learning Wi-Fi signals have been suggested to perform robot (or mobile node) localization. 
Hsich et al. \cite{hsieh2019deep} used deep learning for indoor localization using received signal intensity and channel state information to provide an adaptable solution. 
%Abbas et al. \cite{abbas2019wideep} introduced WiDeep, a Wi-Fi-based indoor localization system using deep learning. 
%Hoang et al. \cite{hoang2019recurrent} employed recurrent neural networks for accurate RSSI indoor localization, and 
Li et al. \cite{li2022self} proposed self-supervised monocular multi-robot relative localization using efficient deep neural networks. However, such learning-based methods need huge volumes of labeled data, often obtained through a dedicated fingerprinting phase, sampling the signals from many APs in the environment (which is also not real-time since scanning all APs in the range takes significantly longer time than obtaining the RSSI of one connected AP). Deep learning-based solutions may frequently experience overfitting to that specific environment or generalization problems in situations with little data. A recent survey on multi-robot relative localization techniques \cite{chen2022survey} %\cite{nessa2020survey,chen2022survey,obeidat2021review}
highlighted the potential of machine learning for localization and the need for more effective and reliable algorithms.

Alternatively, GPR-based active learning approaches have been proposed to perform localization and learning of Wi-Fi signals while overcoming the drawbacks of other learning-based methodologies \cite{fink2010online}. 
Efficient GPR-based robot localization is proposed in \cite{xu2014gp} by using a subset of sampled RSSI observations to reduce the GPR training complexity.
He et al. \cite{he2019calibrating} introduced a calibration method for multi-channel RSS observations using GPs. 
%In contrast, Liang and Liu \cite{liang2019automatic} presented an automatic site survey approach for indoor localization using a smartphone.
Quattrini-Li et al. \cite{quattrini2020multi} applied a GPR-based method for learning the multi-robot communication map of an indoor environment.
%Combining the benefits of deep learning with GPR, Zhang et al. \cite{zhang2019wireless} suggested wireless indoor localization using convolutional neural networks and GPR. 
%While enhancing the scalability of GPR-based approaches, Wang et al. \cite{wang2020indoor} studied indoor radio map creation and localization with deep Gaussian processes, nevertheless confronting computing difficulties. 
Xue et al. \cite{xue2019locate} developed a modified error GPR (MEGPR) to improve the accuracy of device localization using GPR, which can be applied to AP localization. But, it required significant offline fingerprint overhead.
GPR-based techniques, however, can be computationally costly, particularly in high-resolution and large search areas. 

Contrary to existing works, we overcome the limitations of optimization and learning-based approaches by proposing an active learning-based relative localization, which significantly reduces the computational complexity of using the GPR model to locate a signal source (WiFi AP) through hierarchical inferencing. By performing AP-oriented algebraic operations, we obtain the local positions of other robots into the robot's local frame of reference, enabling precise relative localization. 
% Our method uses the GPR model's smoothness and structure through hierarchical inferencing, ensuring accurate localization while avoiding local optimum conditions. Even in situations when robots are not in direct line of sight or have differing beginning orientations, the suggested vector transformation method enables reliable relative localization.
HGP-RL provides a distributed and efficient solution, allowing for a more practical deployment and real-time localization output. Our method is unique in that it can perform relative localization using a single AP (anchor), in contrast to works requiring multiple APs \cite{miyagusuku2018data}.
%for relative localization using sensory data, departing from the limitations of existing state-of-the-art approaches.
Note the proposed HGP is efficient for a source search but not intended for an environment modeling objective typically pursued in the informative path planning works \cite{quattrini2020multi}.

\section{Problem Formulation}
\label{sec:problem}
Assume that a set of $R$ robots are connected to a single (fixed) Wi-Fi AP\footnote{It is possible to extend this work to multiple AP \cite{miyagusuku2018data} (or roaming between APs) setting for large environments, as long as two neighboring robots are connected to at least one common AP in their vicinity.}. Let $_{i}\mathbf{p}^i$ represent the robot's position expressed in its own frame of reference. Here, the subscript on the left of the position variable denotes the frame of reference, and the superscript to the right denotes which object the position refers to. The problem is to obtain relative positions $_{i}\mathbf{p}^j$ for all neighbor robots $j \in R \setminus i $ w.r.t. $i$. We assume a 2D planar environment for simplicity ($d=2$).

Each robot $i$ measures the Received Signal Strength Indicator (RSSI) from the AP (which is ubiquitously available in all modern Wi-Fi devices). We assume that the robots can share information (e.g., their odometry position and predicted AP position) between robots only expressed in their internal frames of reference (i.e., there is no global reference frame available for any robot, so the shared information is not useful for robots directly without transformations). Accordingly, the goal is to develop a distributed method for each robot $i$ to solve this relative localization problem.

\begin{figure}[t]
\centering
%\vspace{-7mm}
 \includegraphics[width=0.98\linewidth]{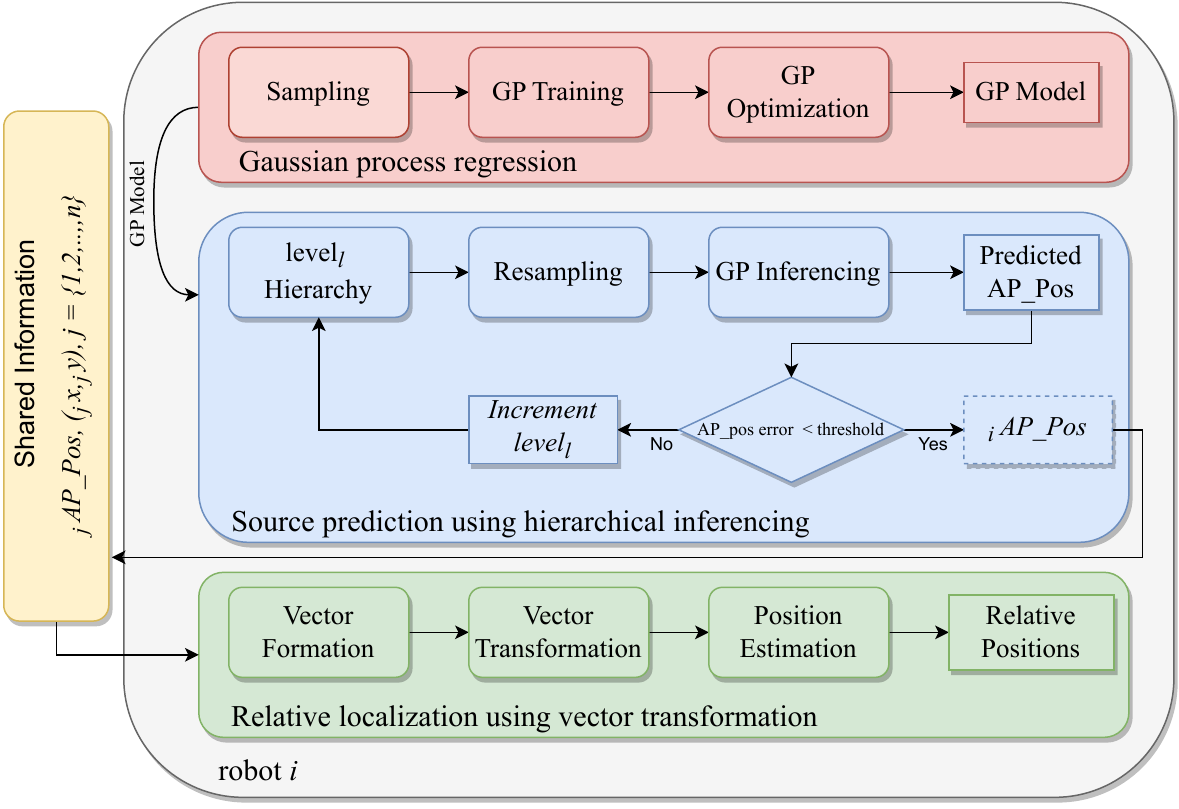}
 \caption{A distributed system architecture of HGP-RL for robot $i$ to relatively localize other robots $j$ using hierarchical GP inferencing of the AP location.}
 \label{architecture}
 %\vspace{-4mm}
\end{figure}
%\fi

\section{Proposed HGP-RL Approach}
\label{sec:gprl}
%We first present our hierarchical inferencing over GPR for accurate source position prediction.
In the proposed distributed architecture, each robot has three major components: the GPR model to train (optimize) and predict the RSSI map, AP location search using hierarchical inferencing over the GPR model, and AP-oriented transformation for relative localization.
% architecture figure removed for brevity 
Fig.~\ref{architecture} delineates the distributed system architecture of the HGP-RL.

\subsection{RSSI-based GP regression}
\label{sec:gpr}
%We model the RSSI values measured along the robot's current trajectory as a Gaussian distribution and use the GPR to forecast the RSSI values to any given value on a boundary around the robot. 
%Due to their accessibility and simplicity, RSSI-based localization techniques have been widely used in indoor and outdoor applications.
%However, environmental changes, multipath propagation, and signal fading can all impact how well these techniques perform. To address the difficulties posed by RSSI-based localization, GPR offers a versatile and reliable method for modeling intricate interactions between inputs and outputs \cite{yiu2017wireless}. 
Wi-Fi RSSI has been used to apply GPR to radio mapping and localization applications successfully \cite{elgui2020learning,fink2010online}, as we can model the noisy RSSI readings as a Gaussian distribution. The technique has proven successful in figuring out the spatial distribution of RSSI values and estimating the user's location based on measured data. 
%Using a single Wi-Fi AP, we extend the application of GPR to the challenge of relative localization among multiple robots in our system.
%\subsubsection{RSSI-based GPR}
Let $\mathbf{X} = {\mathbf{x}_1, \ldots, \mathbf{x}_M}$ represent the $N$ positions within the environment where a robot has measured RSSI values, and $\mathbf{Y} = {y_1, \ldots, y_M}$ correspond to the RSSI measurements acquired at these positions, and the available dataset $\mathcal{D} = {(\mathbf{x}_q, y_q)},{q=1 ... M}$.
%and 

Gaussian Processes \cite{quinonero2005unifying} is a non-parametric probabilistic method to model a random process through few observations. It has a mean function and a covariance function. We use a GPR model to predict the RSSI values at any given position in the surrounding of the robot by leveraging the learned associations between the positions and their respective RSSI values \cite{fink2010online}.
The mean function $m(\mathbf{x})$ of a GP captures the expected value of the function at a given input $\mathbf{x}$. In our case, we choose a constant mean function \cite{quattrini2020multi} to represent the expected RSSI value at any position $m(\mathbf{x}) = \mu$, where $\mu$ is a constant representing the average RSSI in the environment.

The kernel function $k(\mathbf{x}, \mathbf{x}')$ of a GP defines the covariance between the function values at different input points $\mathbf{x}$ and $\mathbf{x}'$. In our case, we use the popular squared exponential (SE) kernel \cite{fink2010online}, also known as the Radial Basis Function (RBF) kernel, which measures the similarity between the positions based on their Euclidean distance:
\begin{equation}
k(\mathbf{x}, \mathbf{x}') = \sigma_f^2 \exp\left(-\frac{||\mathbf{x} - \mathbf{x}'||^2}{2l^2}\right),
\end{equation}
where $\sigma_f^2$ is the signal variance, $l$ is the length scale parameter, and $||\mathbf{x} - \mathbf{x}'||^2$ is the squared Euclidean distance between $\mathbf{x}$ and $\mathbf{x}'$. This kernel function encodes the assumption that the RSSI values at nearby positions are more correlated than those at distant positions. The kernel parameters $\theta = (\sigma_f,l)$ are learned with a dataset of training samples by finding $\theta$ that maximizes the observations' log-likelihood (i.e., $\theta^* = \argmax_{\theta} log P(Y|X,\theta)$). 
After learning, the posterior mean and variance of RSSI prediction for any test location  $q_\star$ are 
%\begin{equation}\mu_{z'}\left( {q}, {q'}, z\right)=k\left({q'}, {q}\right) k(q, q) z \label{eqn:mean} \end{equation}
\begin{equation}
\mu_{\mathcal{M}} [q_{\star}]=m({q})+{k}_{\star}^{{T}}\left({K}+\sigma_{n}^{2} {I}\right)^{-1}({y_q}-m({q})) ,
\label{eqn:mean}
\end{equation}
% \begin{equation} \sigma_{z '}^{2}\left({q},{q'}\right)=k\left({q'}, {q'}\right)-k\left({q'},  {q}\right) k({q}, {q})^{-1} k\left({q}, {q'}\right) \label{eqn:var}\end{equation}
\begin{equation}
\sigma_{\mathcal{M}}^2\left[q_{\star}\right]={k}_{\star \star}-{k}_{\star}^{{T}}\left({K}+\sigma_{n}^{2} {I}\right)^{-1} {k}_{\star} .
\label{eqn:var}
\end{equation}
Here, K is the covariance matrix between the training points $x_q$, $k_*$ is the covariance matrix between the training points and test points, and $k_{**}$ is the covariance between only the test points.
Readers are referred to \cite{fink2010online,xu2014gp} for more information on training and using the GPR model for predicting the RSSI values in non-sampled locations.

The worst-case time complexities for training the GPR model and inferencing one test point with the GPR model are $\mathcal{O}(M^3)$ and $\mathcal{O}(M^2)$, respectively. This is due to the need to invert the $M \times M$ matrix, $K + \sigma_n^2I$ in Eq.~\eqref{eqn:mean} and \eqref{eqn:var}, where $M$ is the size of the training dataset $\mathcal{D}$.
More data points can be added to the dataset, and the GPR model can be re-trained if needed.
In our approach, we assume to obtain the training data of each robot initially using a random walk.
We use the subset sampling (sparse GP \cite{xu2014gp}) to further reduce the training complexity. 
Moreover, training of the GPR model need not be as frequent as inferencing, which must be run in real-time for using the obtained location information.

\begin{figure}
    \centering
    \includegraphics[width=1\linewidth]{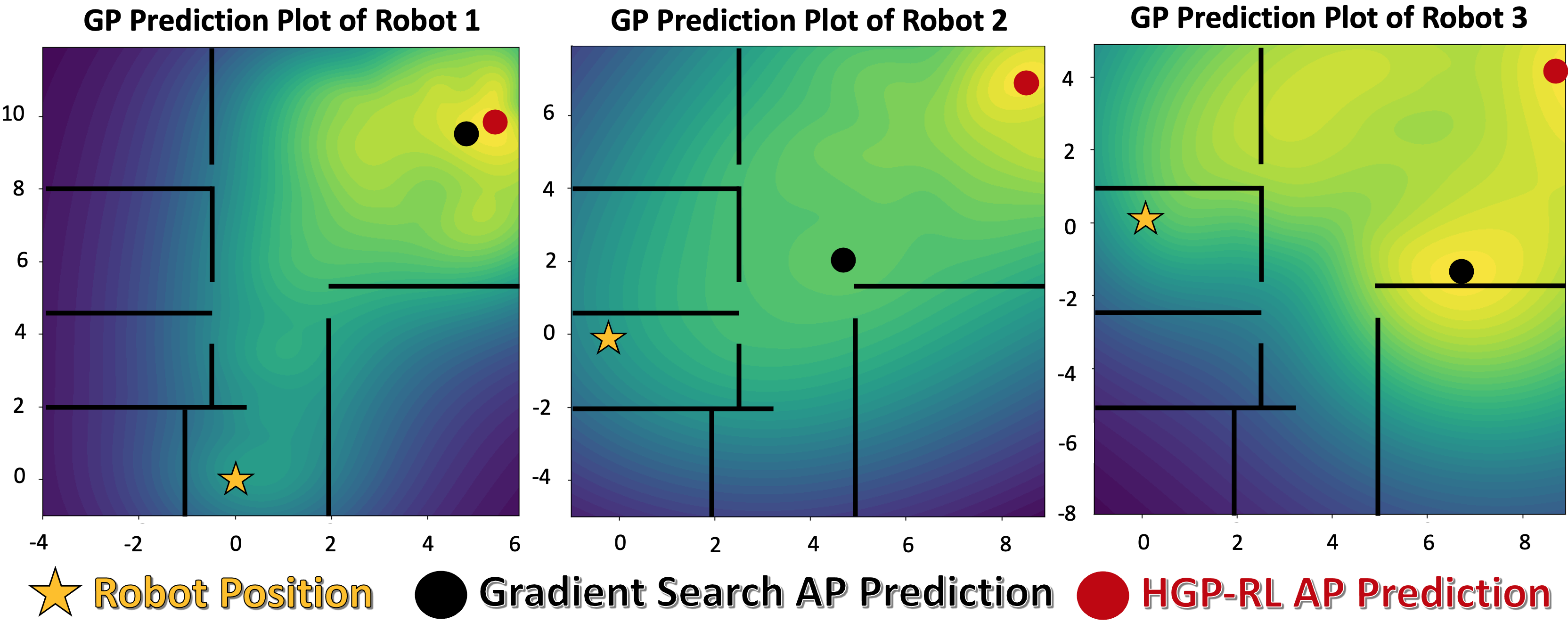}
    \vspace{-6mm}
    \caption{AP predictions for gradient search and hierarchical inferencing.}
    \label{fig:ap_prediction_plots}
    \vspace{-4mm}
\end{figure}
\subsection{AP position Prediction}
%\subsubsection{Source Position Prediction}
\label{sec:source_prediction}
%\textit{Why hierarchical inferencing?}
Traditional methods for locating Access Points (APs) in complex environments often rely on gradient search techniques. These approaches attempt to navigate the search space by iteratively moving toward the direction of the steepest ascent (or descent), aiming to find the global maximum (or minimum) that represents the optimal AP position \cite{hassani2017gradient}. However, real-world scenarios often present significant challenges to the gradient search methodology, primarily due to multiple local maxima in the signal strength landscape. Fig.~\ref{fig:ap_prediction_plots} shows examples of this problem using the real-world data of RSSI maps collected in our laboratory environment (discussed in Sec.~\ref{sec:realworld}). This local phenomenon can misleadingly appear as optimal solutions, causing the typical gradient search technique to locate the source (converge) prematurely and fail to predict the AP's location accurately.
%(results can be seen in Fi.~\ref{fig:ap_prediction_plots})

%TODO
%200 samples. training
%20 iterations optimizing 
\begin{table}[ht]
\vspace{-2mm}
\caption{AP prediction RMSE (m).}
\label{tab:ap_prediction_comparison}
\vspace{-4mm}
\centering
\begin{tabular}{lccc}
\hline
Approach & Robot1 & Robot2 & Robot3 \\ \hline
HGP (Ours) & 0.09 & 0.11 & 0.14 \\
Gradient Search & 0.351 & 6.28 & 5.93 \\ \hline
\end{tabular}
\vspace{-2mm}
\end{table}
On the other hand, a more thorough and effective search of the solution space is made possible by the hierarchical inferencing technique, which uses the problem's multi-resolution structure. We have performed experiments on real-world scenarios and observed that gradient search fails to predict AP position when the initial search position is 1m away from the actual AP. Furthermore, Table~\ref{tab:ap_prediction_comparison} has shown that gradient search has a high prediction error (0.351m - 5.93m) for a space of 10m by 11m. Hence, Gradient search methods are susceptible to local maxima, which the proposed hierarchical inferencing approach effectively avoids by iteratively refining the search from a coarse to a fine resolution. 
Thus, the hierarchical inferencing approach in HGP addresses the critical shortcomings by introducing a scalable, efficient, and robust multi-resolution search strategy. This method mitigates the issue of local maxima and ensures high precision in AP position prediction, thereby significantly enhancing the reliability and applicability of WiFi-based localization techniques in real-world scenarios.

%\textit{Our approach: Hierarchical inferencing}
For each robot $i$, using the trained GP model $\mathcal{M}$ with the mean function $m(\mathbf{x})$ and the kernel function $k(\mathbf{x}, \mathbf{x}')$, we can predict the Wi-Fi AP position $_{i}\mathbf{p}^{*AP}$ in the frame of the robot $i$ by finding the position that maximizes the posterior distribution of the source position given the data $\mathcal{D}$:
\begin{equation}
{}_{i}\mathbf{p}^{*AP} = \argmax_{\mathbf{p} \in {}_{i}\mathcal{B}} \mu_{\mathcal{M}}(\mathbf{p} ) 
%p(\mathbf{p} \mid \mathcal{D}, m(\mathbf{x}), k(\mathbf{x}, \mathbf{x}')),
\label{eqn:predict}
\end{equation}
%where $p(\mathbf{p} | \mathcal{D}, m(\mathbf{x}), k(\mathbf{x}, \mathbf{x}'))$ represents the posterior distribution of the source position conditioned on the data $\mathcal{D}$, and 

Here ${}_{i}\mathcal{B}$ is the boundary of the search space centered around the robot $i$ and the posterior mean $\mu_{\mathcal{M}}(.)$ for a given grid point in ${}_{i}\mathcal{B}$ is calculated using Eq.~\eqref{eqn:mean}. The associated uncertainty (prediction variance) is $\sigma_{\mathcal{M}}^2 ({}_i \mathbf{p}^{*AP})$ from Eq.~\eqref{eqn:var}.
%We assume a fixed boundary size of the RSSI map for all robots.
%$_{i}\mathbf{p}^{*AP}$ considered as initial source position estimation, which can be improved will be improved hierarchical inferencing.
%In our proposed solution, the GP model is trained by taking RSSI samples at random points and then optimizing the model for better inferencing. The GP model learns the link by incorporating the mean function, kernel function, and the collected data.

%GPR involves training the hyperparameters of the GPR model and inferencing using the trained GPR model for test data points. Therefore, we focus on reducing the inferencing complexity in this work, as we must regress the RSSI values over a large workspace to locate an AP. 

The computational complexity of calculating Eq.~\eqref{eqn:predict} considering the full search space consisting of $N$ grid points in each dimension of the 2D search space in ${}_{i}\mathcal{B}$ is $\mathcal{O}((NM)^2)$, where $M$ is the size of the training data. Usually, $N >> M$ for a source localization problem, and since $M$ can be fixed for a given GPR model, the worst-case search complexity can be generalized as $\mathcal{O}(N^2)$.
The complexity of this search will exponentially scale with $N$ (i.e., very high for a fine resolution of search space grids). For instance, for a 20m x 20m workspace with a reasonable centimeter-level resolution of finding the AP location around the robot, the search space of this dense-resolution ${}_{i}\mathcal{B}$ has approx. 40K test points. Applying GPR inferencing on each of these test points could severely plunge the computational efficiency, limiting its applicability to robots with low computational resources.

Therefore, we propose a hierarchical inferencing strategy to make the search for AP more efficient. We follow a multi-resolution search technique, which refines the search for the ideal AP position from coarse to fine \cite{yin2019mrs}. We can quickly determine the position that maximizes the posterior distribution of the AP position while taking advantage of the structure and smoothness of the GPR model by combining hierarchical inferencing with GPR.
%In the GPR model, we aim to determine the position ${_i}\mathbf{p}^{*AP}$ that maximizes the posterior distribution of the AP position. 

The search process begins with the coarsest inferencing level with a low resolution $r_1$ of search space centered around the robot $i$. At each level, we apply GPR for inferencing. Let us denote the AP location found in the first level as $\mathbf{{}_i\mathbf{p}^{AP}_1}=({}_i{x}^{i} + r_1 {}_i{x}^{*AP}_1, {}_i{y}^{i} + r_1 {}_i{y}^{*AP}_1)$, which is the coarsest estimate of the AP position. Here, ${}_i{x}^{i}$ is the reference (origin) location in the lowest resolution, and the AP location $({}_i{x}^{*AP},{}_i{y}^{*AP})$ in this level is found by Eq.~\eqref{eqn:predict} as ${}_i{p}^{*AP}_1 = \argmax_{\mathbf{p} \in {}_{i}\mathcal{B}_k} \mu_{\mathcal{M}}(\mathbf{p} ) $. 
%At each level, we apply GPR for inferencing. 
%We insert these values into our Gaussian Process (GP) regression model, which in mathematical terms, gives us $p(\mathbf{p} \mid \mathcal{D}, m(\mathbf{x_1}), k(\mathbf{x_1}, \mathbf{x_1}'))$. This posterior distribution represents the likelihood of our AP position given our first-level data.
Similarly, in level 2, we refine the resolution $r_2$ and center the search around $\mathbf{_ip^{AP*}_1} \pm \sigma_{\mathcal{M}(_ip^{AP*}_1)}$ (considering the uncertainty from the previous level) and now we obtain $\mathbf{_i\mathbf{p}^{AP*}_2}=({}_i{x}^{AP}_1 + r_2 {}_i{x}^{*AP}_2, {}_i{y}^{AP}_1 + r_2 {}_i{y}^{*AP}_2)$. %The term 'res' indicates our resolution refinement at this stage. 
%This yields the following posterior distribution: $p(\mathbf{p} \mid \mathcal{D}, m(\mathbf{x_2}), k(\mathbf{x_2}, \mathbf{x_2}'))$. 
This is a more refined likelihood of our AP position based on the two-level inferencing.
The process can be continued to search the AP at even finer levels until we reach the maximum levels in the hierarchy.
%At level 3, we further refine the resolution and have $\mathbf{x_3}=(_i{x}_1+_i{x}_2+res,_i{y}_1+_i{y}_2+res)$. This provides the posterior distribution as $p(\mathbf{p} \mid \mathcal{D}, m(\mathbf{x_3}), k(\mathbf{x_3}, \mathbf{x_3}'))$, which is an even finer likelihood of our AP position.
In the end, the final AP position is obtained by summing the data from $K$ levels of the hierarchy: 
\begin{equation}
\begin{aligned}
& {}_{i}\mathbf{p}^{*AP} = ({}_i{x}^{i} + \sum_{k=1}^K r_k {}_ix^{AP}_k, {}_i{y}^{i} + \sum_{k=1}^K r_k  {}_iy^{AP}_k ) , \\
& \text{where, } {}_ix^{AP}_k = \argmax_{\mathbf{x} \in {}_{i}\mathcal{B}_k} \mu_{\mathcal{M}}(\mathbf{x} ) 
\end{aligned}
\label{eqn:predict-hier}
\end{equation}
%$\mathbf{x}= (_i{x}_1+_i{x}_2+_i{x}_3,_i{y}_1+_i{y}_2+_i{x}_3)$. Plugging this data into the GP model gives us the final posterior distribution $p(\mathbf{p} \mid \mathbf{x}, m(\mathbf{x}), k(\mathbf{x}, \mathbf{x}'))$.
%Finally, we can state that the AP position that maximizes this final posterior distribution is our desired ${_i}\mathbf{p}^{*AP}$, which will be used in Eq.\ref{eqn:predict}.
% \begin{equation}
% {i}\mathbf{p}^{*AP} = \arg\max{\mathbf{p}} p(\mathbf{p} \mid \mathbf{x}, m(\mathbf{x}), k(\mathbf{x}, \mathbf{x}')),
% \label{eqn:predict}
% \end{equation}
The uncertainty associated with this hierarchical prediction can also be calculated as
\begin{equation}
\begin{aligned}
& {}_{i}\mathbf{\sigma_{\mathcal{M}}^2}({}_i \mathbf{p}^{*AP}) = \sum_{k=1}^K r_k \sigma_{\mathcal{M}}^2 ({}_i \mathbf{p}^{AP}_k),
\end{aligned}
\label{eqn:predict-hier-std}
\end{equation}
In essence, the hierarchical inferencing strategy finds a position estimate at each level $k$ inside a boundary space of ${}_{i}\mathcal{B}_k$. It ultimately yields a position that maximizes the posterior distribution, ensuring high precision in predicting the AP position at each robot.
The search complexity of this hierarchical inferencing approach can be generalized as $\mathcal{O}(K*(\lambda N)^2)$, where $\lambda << 1$ is a scaling factor of the full resolution GPR if $N$ is the number of grid points in a high-resolution RSSI map. 
%For simplicity, we assume the size of ${}_{i}\mathcal{B}_k}$ in each level of HGP-RL is the same.
%We may efficiently use the smoothness and structure of the GP model by using hierarchical inferencing over GP regression, which enables the precise prediction of the Wi-Fi source position. 
%the robots can localize themselves with great accuracy.

\iffalse
The following sentences sum up the hierarchical inferencing algorithm:

\begin{enumerate}
\item In the frame of reference of robot $i$, initialize the search grid with a coarse resolution.
\item Initialize prediction threshold $\mathcal{T}$ to ensure convergence.
\item For each grid point $\mathbf{p}$, compute the posterior distribution as Eq.~\ref{eqn:predict} using the GP model discussed in Section.~\ref{sec:source_prediction}.
\item Find the grid point $\mathbf{p}$ with the highest posterior probability.
\item Generate a finer sampling grid centered around $\mathbf{p}$.
\item Repeat steps 2-4 until the source position prediction error meets the prediction threshold $\mathcal{T}$.
\item The final $\mathbf{p}$, is the estimated AP position $_{i}\mathbf{p}^{*AP}$.
\end{enumerate}
\fi

We now theoretically analyze how the hierarchical search process in Eq.~\eqref{eqn:predict-hier} approximates the optimal search in Eq.~\eqref{eqn:predict}.
\begin{lemma}
\label{lemma:hierarchy}
The hierarchical inferencing in Eq.~\eqref{eqn:predict-hier} yields AP position predictions with an approximation within a threshold $\epsilon$ of the optimal estimate from a dense resolution GPR RSSI prediction map (Eq.~\eqref{eqn:predict}).
\end{lemma}
\begin{proof}
Let's define $\mathbf{p}_{dense}$ and $\mathbf{p}_{hier}$ to be the AP position as predicted by the dense resolution GPR and the hierarchical inferencing, respectively. Suppose the hierarchy of the GPR model has $K$ levels, and each level is inferred with a density resolution of $r_j$ at level $j$, where $1 \leq j \leq K$ and $r_1 < r_2 < ... < r_K$. Let's further assume that we use the same pre-trained GPR model in both cases. %, correctly representing the uncertainty in their predictions.
Our goal is to show that as we move deeper into our hierarchical GPR model, we use the GPR model with increasing density resolutions to predict the AP position. The prediction error should decrease with each level, and this error will ultimately converge to a value less than $\epsilon$.
%after a sufficiently large number of iterations $t$. 
Formally, for an infinitesimally small threshold $\epsilon > 0$, we can observe that there exists a level $k$ such that we have $||\mathbf{p}_{hier}(k) - \mathbf{p}_{dense}|| < \epsilon$. 
%The magnitude of this difference decreases with $k$ since, as $k$ increases, we move to a denser level in the hierarchical model, thereby getting a more accurate prediction. 
\end{proof}

The above lemma signifies that hierarchical inferencing over GPR allows for significantly improved computational efficiency while keeping the accuracy close to an optimal outcome of a dense-resolution GPR inferencing map.

\iffalse
\begin{algorithm}[tb]
\SetAlgoLined
\For{robot $i$ in the set of connected robots $R$}{
\textbf{Input}: $M$ initial training samples (e.g., from a random walk) with a GPR training dataset $\mathcal{D} = {({}_i\mathbf{p}^i_k, {}_irssi_k)}$, where ${{}_i\mathbf{p}^i_m=1 ... M}$ are robot positions (odometry) in local frames, and $rssi_m$ are the corresponding Wi-Fi RSSI values. 

\textbf{Output}:  ${}_i\mathbf{p}^{AP}$ and $_i\mathbf{p}^{j}$, the position of AP and all neighbor robots $j$ in the local frame of robot $i$. 

\Begin{
Train/Re-train GP model $\mathcal{M}$ with training samples $\mathcal{D}$, mean function $m(\cdot)$, and kernel function $k(\cdot, \cdot)$\;
Perform hierarchical inferencing (Eq.~\eqref{eqn:predict-hier}) to find the AP position $_i\mathbf{p}^{AP}$\;
\For{neighbot robot $j$}{
    Send $_i\mathbf{p}^{AP}$ and ${}_i\mathbf{p}^i$ to robot $j$\;
    Receive $_j\mathbf{p}^{AP}$ and ${}_j\mathbf{p}^j$ from robot $j$\;
    %Compute the transformation matrix $_i\mathbf{T}^j$ using the initial robot orientations and robot positions in local frames\;
    Compute ${}_i\mathbf{p}^j$, the relative position of robot $j$ in the robot $i$'s local frame using AP-oriented transformation in Eq.~\eqref{eqn:vector-transform}\;
}
%\textbf{return} $_i\mathbf{p}^{j}$ and $_i\mathbf{p}^{j}$\;
}
}
\caption{Distributed Hierarchical Gaussian Processes for Relative Localization (HGP-RL)}
\label{alg:gp_relative_localization}
\end{algorithm}
\fi

\subsection{AP-oriented relative localization}
\label{sec:relative_localization}
We assume that each robot $i\in R$ can predict the Wi-Fi AP position ${}_i\mathbf{p}^{AP}$ (using the hierarchical GPR map of the RSSI readings, as described in Sec.~\ref{sec:source_prediction}) and obtain its current location ${}_i\mathbf{p}^{i}$ (e.g., using odometry and IMU) in its internal frame of reference $i$ and can share these data with other robots.
%With this configuration, we can use vector transformation to determine where robot $j$ is about robot $i$.
%\subsection{Vector Transformation}
%Let $_{i}\mathbf{p}^{AP}$ represent the position of the Wi-Fi AP as predicted by robot $j$ in the frame of reference of robot $i$, and let $_{j}\mathbf{p}^{AP}$ denote the position of the Wi-Fi AP as predicted by robot $j$ in its frame of reference. 
Leveraging the fact that the robots are locating the same (non-moving) AP, we can infer that 
\begin{equation}
({}_{i}\mathbf{p}^j - {}_{i}\mathbf{p}^{AP}) = {}_{i}\mathbf{R}^j ({}_{j}\mathbf{p}^j - {}_{j}\mathbf{p}^{AP}) .
\label{eqn:vector-AP}
\end{equation}
Here, $_{i}\mathbf{R}^j$ denotes the rotation matrix that transforms a vector from robot $j$'s frame to robot $i$'s frame.
Eq.~\eqref{eqn:vector-AP} indicates that the vector expressing the line between a robot and AP should be equal in two different frames of reference as long as we apply rotation between such frames. 
%In the given system, robot $j$ shares $_{j}\mathbf{p}^{AP}$ to robot $i$ along with its current position $_{j}\mathbf{p}$. Furthermore, let $_{i}\mathbf{p}^{AP}$ represent the position of the Wi-Fi AP as predicted by robot $i$ in its frame of reference.
% , and let $_{i}\mathbf{t}^j$ represent the translation vector between the origins of robot $i$ and robot $j$. 
% confusing brevity %Assuming a 2-D space, we define a transformation matrix $_{i}\mathbf{T}^j$ = $_{i}\mathbf{R}^j$.
% \begin{equation}
% _{i}\mathbf{T}^j = \begin{bmatrix}
% _{i}\mathbf{R}^j & _{i}\mathbf{t}^j \
% \mathbf{0} & 1
% \end{bmatrix},
% \label{eqn:transformation}
% \end{equation}
% Where $\mathbf{0}$ is a row vector of zeros with the same dimension as $_{i}\mathbf{t}^j$.
By applying AP-oriented algebraic transformation (Eq.~\eqref{eqn:vector-AP}), a robot $i$ can obtain the relative position of a robot $j$ using% from $j$'s frame to $i$'s frame using
\begin{equation}
_{i}\mathbf{p}^{j} = {}_{i}\mathbf{p}^{AP} + {}_{i}\mathbf{R}^j (_{j}\mathbf{p}^{j} - {}_{j}\mathbf{p}^{AP}).
\label{eqn:vector-transform}
\end{equation}
%Thanks to this transformation in Eq.~\eqref{eqn:vector-transform}, every robot $i$ can now locate every neighbor robot $j$ in its local frame of reference. 
A caveat here is that the rotation matrix for every neighbor robot should be known.
%TODO%
We posit that robots know their initial orientations (to obtain $_{i}\mathbf{R}^j$), with a reasonable assumption that robots start from a command station in practice, where initial configuration can be controlled. Alternatively, they can also obtain this in real-time using their magnetometers in the IMU (magnetic heading acting as a proxy for global orientation). This is useful for aerial robots but not ground robots, given that a magnetic field is sensitive to man-made structures, especially in urban environments. Nevertheless, this limitation can be alleviated by the use of Angle-of-Arrival estimation techniques \cite{jadhav2022toolbox} or RSSI-based relative bearing estimation \cite{parasuraman2019consensus,parashar2020particle} for the relative bearing of neighboring robots. Precise estimation of this rotation matrix is out of the scope of this paper.

%The assumption here is that the robots know their relative orientations (to obtain the rotation matrices transforming from one robot to another). This can be practically achieved by setting/knowing the initial orientations of the robots. Alternatively, this assumption can be relaxed by combining HGP-RL with Angle-of-Arrival estimation techniques \cite{jadhav2022toolbox} or RSSI-based relative bearing estimation \cite{parasuraman2019consensus}.

Considering all neighbor robots, the time complexity of this relative localization component is a constant $\mathcal{O}(|R|-1)$. Coupled with a low memory and communication complexity, HGP-RL enables a practical solution for relative localization.

\section{Experimental Validation}
\label{sec:experimental_setup}

We use the Python-based Robotarium \cite{wilson2021robotarium} simulator\footnote{\url{https://github.com/robotarium/robotarium_python_simulator}} as our experiment testbed. %it enables swarm-level behaviors using the obtained relative localization data in the future. Also, the 
It provides remote accessibility to swarm robots, enabling a reliable and regulated setting for carrying out our experiments and fostering the repeatability and confirmation of the findings made in this study.

% The GPR model must be updated, and the relative positions of the robots must be predicted for the proposed localization and learning approach to work. 

\subsection{Experimental Setup}
Robotarium has a $3.2 \times 2$ meter rectangular area. We deploy three ($|R|=3$) robots with unicycle dynamics (which can be converted from a single integrator dynamics) to populate the simulation environment, and each robot's initial position is randomly chosen (but their initial orientation is set the same). 
%The robots may randomly stroll in the predetermined area because they are set up to use a single integrator controller. 
To maintain a controlled experimental setup, virtual limits are added to the simulation environment to prevent robots from straying outside the designated area.

The Robotarium's server communicates with the robots during the tests, enabling the exchange of data necessary for the localization process. The platform gathers and records information about the positions of the robots, as well as AP location forecasts derived from the GPR model. 
% The robots move randomly around the workspace while running Alg.~\ref{alg:gp_relative_localization} for each robot. 

Based on the experimental setup, we performed two sets of experiments: First, we analyzed the performance of the hierarchical inferencing algorithm (Sec.~\ref{sec:source_prediction}) for AP location estimation in each robot. Then, we integrate this algorithm with AP-oriented transformation and analyze the combined relative localization performance of HGP-RL.
% (Alg.~\ref{alg:gp_relative_localization}). 
The experiments are conducted for 300 iterations and 10 trials each.
In our trials, the GPR model is initially trained on ten random samples from a random walk stage. After each iteration, the model is updated, improving its prediction accuracy by incorporating the most recent data collected from the robots. In all experiments, the simulated AP is located at the center of the workspace. We simulated the RSSI from AP to a robot $i$ using 
\begin{equation}
    rssi({}_i\mathbf{p}^i) = rssi_{d0} - 10\eta \log_{10} (\| {}_i\mathbf{p}^i - {}_i\mathbf{p}^{AP}\|) - \mathcal{N}_s(0,\sigma_s^2), % - \mathcal{N}_mp(0,\sigma_m^2)
    \label{eqn:rssi}
\end{equation}
where $rssi_{d0} = -20 dBm$ is the reference RSSI value close to the source, $\eta = 3$ is a path loss exponent, and $\mathcal{N}_s$ is a Gaussian noise to represent shadow fading with a std. $\sigma_s = 2 dBm$. We further added a zero-mean Gaussian noise with a std. $\sigma_{mp} = 2 dBm$ to represent multipath fading effects. These values are set based on typical 2.4GHz Wi-Fi RSSI propagation characteristics in indoor settings \cite{parasuraman2014multi}.

%The resolution settings for different approaches can be seen in Table.~\ref{tab:exp-config}. 

%The performance of the suggested localization approach is assessed through the processing and analysis of this data. 

\iffalse % TODO MAY BE UNNECESSARY?
\subsection{GPR Model}
Based on the available information from the robots' positions and the AP location estimates, a GPR model is used to forecast the source location. 
%A non-parametric probabilistic generalized projection rule (GPR) model can capture complicated relationships between input and output data, offering predictions and estimating the uncertainty involved. 
GPR models are ideally suited for jobs involving erratic data or noisy observations because of this property, which is why they successfully predicted the source position in our trials. 
The GPR model is then updated using this information by adjusting its hyperparameters and adding new data. 
\fi

\iffalse
\begin{table}[t]
\begin{center}
\caption{Hierarchical Configurations }
\label{tab:hier-config}
% \vspace{-4mm}
\resizebox{\linewidth}{!}{
\begin{tabular}{|c|c|c|c|c|}
\hline
\textbf{Simulation Parameters}
& \textbf{L 0} & \textbf{L 1} & \textbf{L 2} &\textbf{L 3}\\
\hline
Resolution & 0.1 & 0.05 & 0.025 & 0.0125\\
\hline
\end{tabular}
}
\end{center}
\vspace{-6mm}
\end{table}
\fi

\begin{figure}[t]
    \centering
%\begin{subfigure}{\linewidth}
\centering
\includegraphics[width=\linewidth]{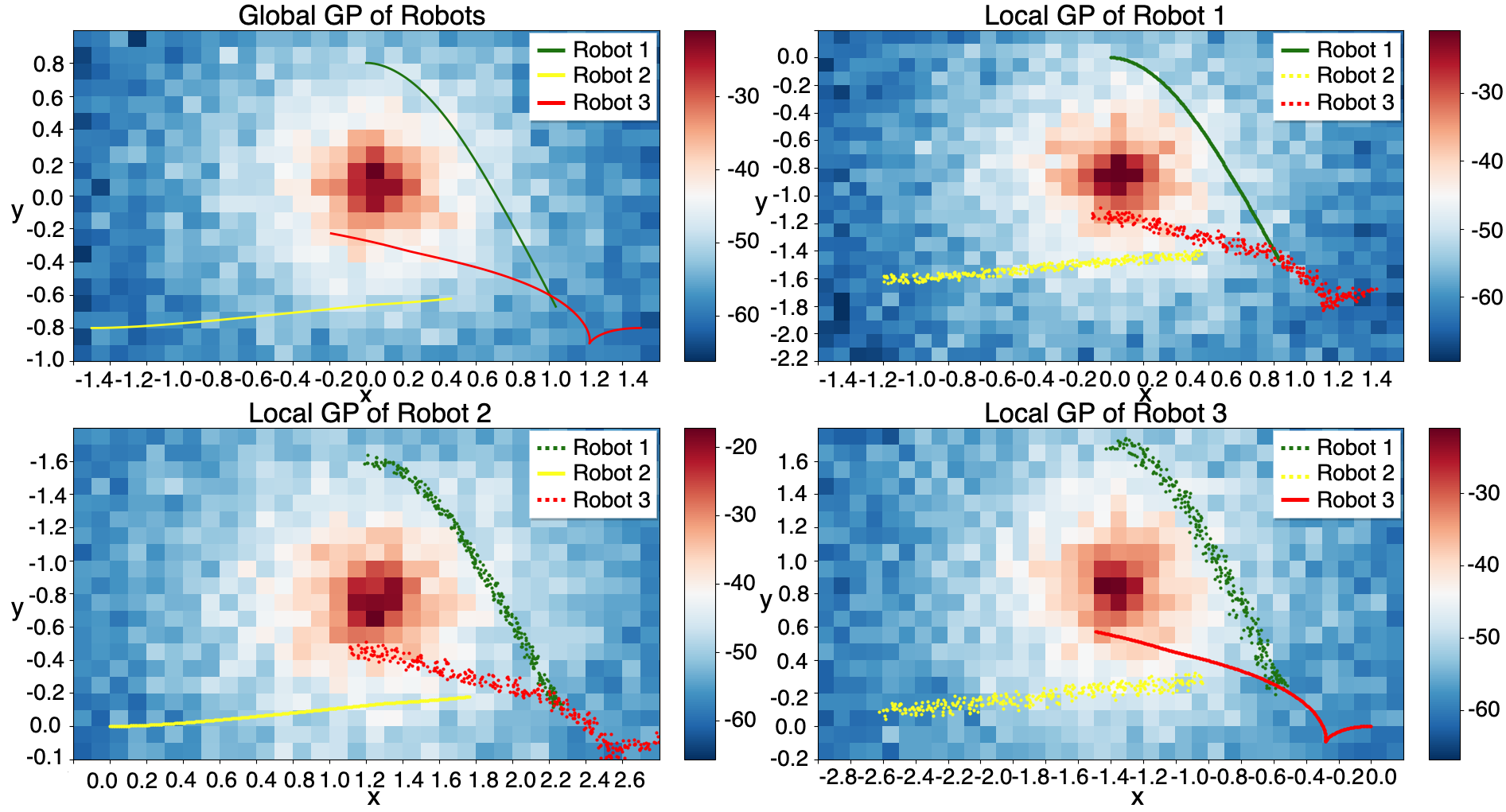}
%\end{subfigure}
%\vspace{-4mm}
\caption{Gaussian processes inferencing for global and local frames along with the ground truth and predicted robot trajectories of HGP-RL.}
    \label{fig:loc_trjectory_plots}
    %\vspace{-4mm}
\end{figure}

%\section{Results and Discussion}
%\label{sec:results}
%Below, we discuss the results of AP location predictions and relative localization performances.

\begin{figure*}[t]
    \centering
%\begin{subfigure}{\linewidth}
\centering
\includegraphics[width=0.99\linewidth]{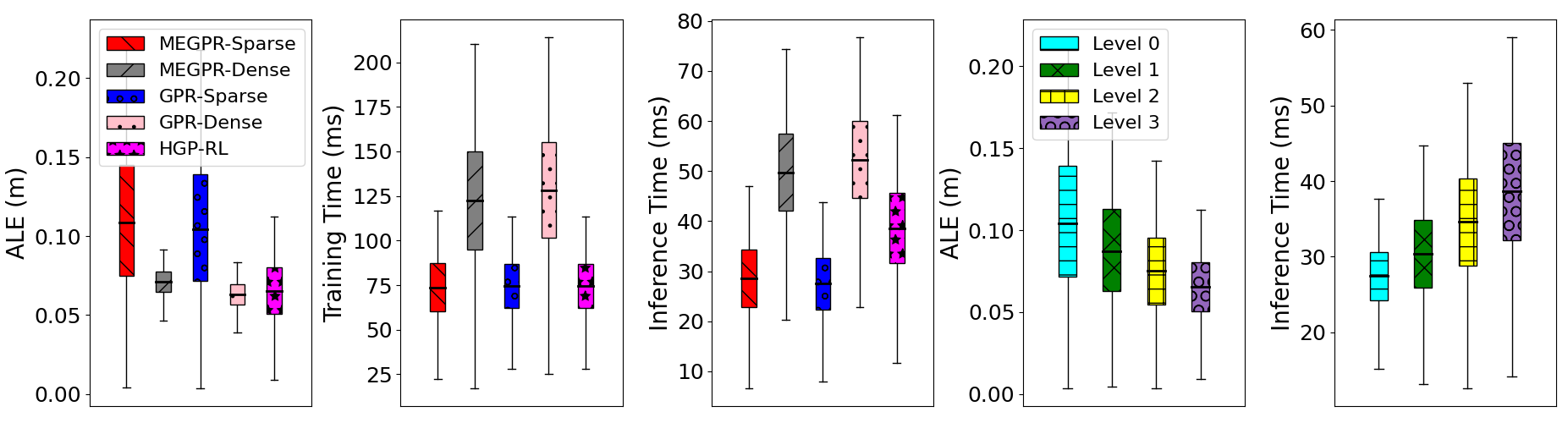}
% \includegraphics[width=0.32\linewidth]{source_train_plot.pdf}
% \includegraphics[width=0.32\linewidth]{source_inference_plot.pdf}
%\end{subfigure}
%\vspace{-2mm}
\caption{AP prediction performance plots (Absolute Localization Error and GPR Training and Inferences Times) of HGP-RL (ours) compared with GPR \cite{quattrini2020multi} and \cite{xue2019locate} (with Sparse and Dense resolutions). The effect of the number of hierarchy levels in HGP-RL is also shown in the right-most plots.}
    \label{fig:source_plots}
\end{figure*}

\subsection{AP Source Prediction}
We compared HGP-RL with two state-of-the-art GPR-based source localization approaches: 1) standard GPR model \cite{quattrini2020multi} to predict the RSSI in a dense resolution map; 2) modified error GPR (MEGPR) proposed in \cite{xue2019locate}. We applied a spatial resolution of 0.1m (Sparse) and 0.0125m (Dense) in both these methods. The proposed hierarchical HGP-RL used a 4-level hierarchy with the following coarse-to-grain resolutions of [0.1m,0.05m,0.025m,0.0125m].
We fix the size of the search space in all levels of the HGP-RL hierarchy as $|{}_{i}\mathcal{B}| = 30\mathrm{x}30$ (equivalent to the Sparse setting), whereas the Dense models will have 240x240 grid points.
Fig.~\ref{fig:loc_trjectory_plots} shows the predicted GPR map at the coarsest resolution and the estimated relative positions in a trajectory of a sample trial conducted with our HGP-RL implementation in ROS.

We compare the performance of HGP-RL with the standard GPR and the MEGPR approaches in terms of the absolute localization error (ALE) of the AP location estimates, time to train the GPR models, and time to obtain inference with AP location estimates with the trained GPR models.

Fig.~\ref{fig:source_plots} shows the results of this experiment. HGP-RL outperformed the competitors by matching the ALE of a dense resolution GPR with less inference time for predictions.
%HGP-RL outperformed both GPR and MEGPR in terms of accuracy and efficiency. The findings show that in terms of source ALE (Average Localization Error), training time, and inference time, HGP-RL with hierarchy outperforms its competitors.
The ALE of AP localization is improved by HGP-RL with hierarchy, with an error roughly 36\% lower than GPR/MEGPR-Sparse and roughly 8\% lower than MEGPR-Dense. 
%This indicates our method's increased accuracy, which is necessary for relative localization jobs. 
The training and inference times for our HGP-RL have been kept reasonably low. HGP-RL takes about 40\% less time to train than both GPR-Dense and MEGPR-Dense, while matching the training time of Sparse resolution models. 
Additionally, HGP-RL balances the inference time between the Sparse and Dense resolution models by exploiting the advantages of the hierarchical inferencing approach. HGP-RL is around 27\% quicker than the Dense approaches but 35\% slower than the Sparse approaches. However, the higher accuracy in ALE justifies this increase in inference time. 
Furthermore, the effect of hierarchy in HGP-RL is also analyzed. The Inference Time is linear with the number of levels, as can be seen in the right-most plots in Fig.~\ref{fig:source_plots}, but the improvement in the ALE exponentially decreases with more hierarchy levels saturating at 4 levels (which we set for the next set of experiments).
%and provides evidence for its effectiveness with improved source prediction accuracy. 
%These findings demonstrate the dominance of the HGP-RL with hierarchy and the possibility of our technique in real-world settings where accuracy and efficiency are key considerations.
These results demonstrate the effectiveness of our HGP-RL because it can deliver quick and precise AP localization.

\iffalse
\subsection{Localization and Learning Method}

The robots estimate the relative positions of their peers using the projected source location after the updated GPR model. Each robot participates in the localization process using a consensus-based methodology by sharing its source location forecast and position with its neighbors. Then, considering the corresponding uncertainties from the GPR model, the robots compute the weighted average of the received predictions and positions. Using the knowledge of their peers, the robots can cooperatively improve their localization estimations using this consensus-based method. The performance of the suggested localization and learning strategy is evaluated by contrasting the robots' anticipated relative positions with their actual positions in the simulated environment. Through iterative learning and localization, robots become more capable of accurately predicting their peers' positions, demonstrating the effectiveness of the proposed method.
\fi

\subsection{Relative Localization Results}
%\subsection{Multi-Robot Simulations}
To validate the accuracy and robustness of the relative localization, we compare the performance of HGP-RL with two state-of-the-art range-based relative localization algorithms for their relevance to our proposed approach: 1) Terrain Relative Navigation (TRN) \cite{wiktor2020ICRA} and 2) Distributed Graph Optimization for Relative Localization (DGORL) \cite{latif2022dgorl}. 

We employ the RMSE of relative localization error metric to assess how well the proposed approach compares with the state-of-the-art methods. The average localization error is quantified by RMSE, which also captures differences between the robots' anticipated relative positions and their actual positions in the simulated environment. 
%Insights into the convergence rate and robustness of the learning approach are shown by analyzing RMSE values throughout the experiments. 
%This reveals how well the learning method works to improve the GPR model and adapt to new data under various initial conditions and random walks. 
%We can confirm the efficacy and dependability of the suggested localization strategy in the Robotarium trials thanks to this thorough study using RMSE. 
We also present the efficiency metrics in terms of communication payload (KB) for data sharing between the robots, CPU utilization (\%), and the processing time (ms) per iteration for each approach.
%to validate the efficiency of the proposed approach.
Fig.~\ref{fig:loc_performance_plot} presents the results of these key metrics, with RMSE plots shown under different shadowing noise levels (in dBm) of simulated RSSI in Eq.~\eqref{eqn:rssi}.

%In this section, we report the findings of our studies, comparing the accuracy (RMSE) and efficiency (CPU consumption and processing time) of the proposed HGP-RL method with the state-of-the-art localization approaches; TRN \cite{wiktor2020ICRA}, DGORL \cite{latif2022dgorl}, and HGP-RL.

% \begin{figure*}[t]
%     \centering
% %\begin{subfigure}{\linewidth}
% \centering
% \hspace{-8mm}
% \includegraphics[width=0.99\linewidth]{error_plot.png}
% %\end{subfigure}
% \caption{Localization performance comparison plot for three robots using TRN \cite{wiktor2020ICRA}, DGORL \cite{latif2022dgorl} and proposed HGP-RL approach }
%     \label{fig:loc_error}
% \end{figure*}

\begin{figure*}[t]
\centering
%\vspace{-5mm}
%\begin{center}
\centering
 \includegraphics[height=0.25\linewidth]{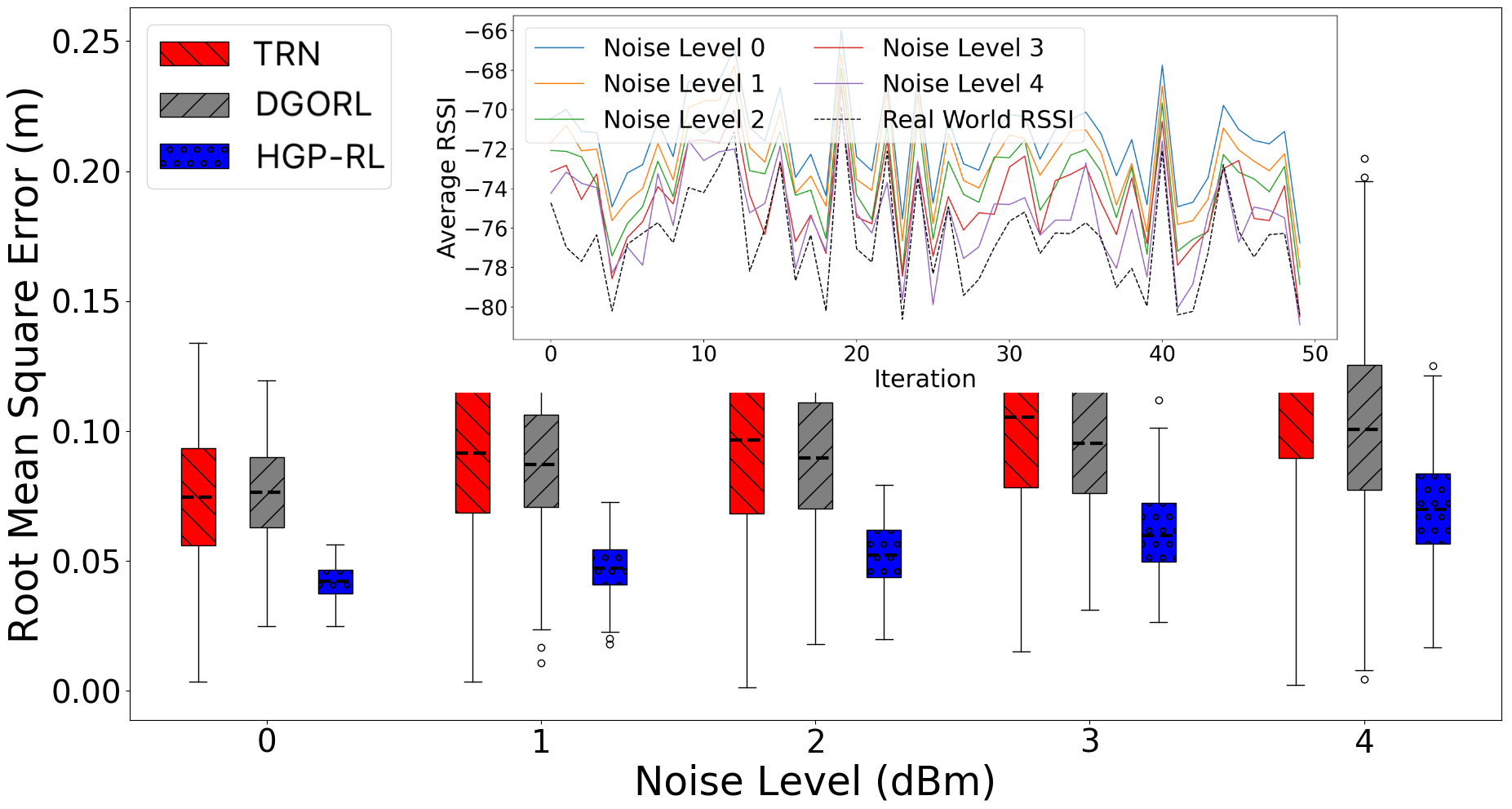}
 \includegraphics[height=0.25\linewidth]{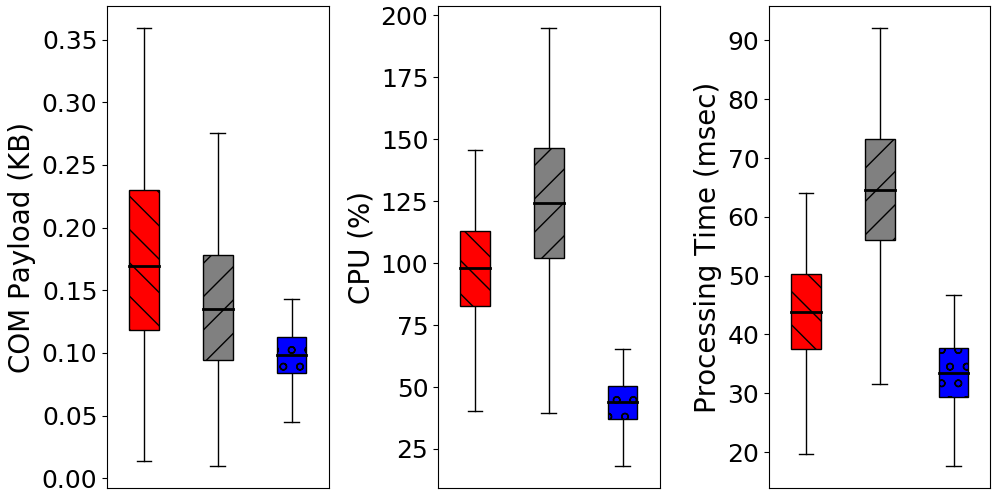}
%\end{center}
%\vspace{-2mm}
 \caption{Results of various performance metrics. From left to right, we show the localization accuracy (RMSE) under different simulated RSSI noise levels (the embedded plot represents the respective RSSI variations), the communication payload, the computation overhead, and the run time performance of various relative localization approaches. It can be seen that the HGP-RL consistently outperforms other approaches in all the metrics. }
 \label{fig:loc_performance_plot}
 %\vspace{-4mm}
\end{figure*}

\iffalse
\begin{figure}[t]
    \centering
    \caption{Computing performance of different localization approaches.}
    \label{fig:loc_computing_plot}
\end{figure}
\fi

\iffalse
\begin{figure}[t]
    \centering
%\begin{subfigure}{\linewidth}
\centering
\includegraphics[width=0.99\linewidth]{rel_loc_plots.png}
% \includegraphics[width=0.32\linewidth]{loc_cpu_plot.pdf}
% \includegraphics[width=0.32\linewidth]{loc_proc_plot.pdf}
%\end{subfigure}
\caption{Localization performance evaluation comparison plots of SOTA approaches per iteration; TRN \cite{wiktor2020ICRA} (Bayesian fusion and collaborative localization), DGORL \cite{latif2022dgorl} (graph optimization and relative localization), and the proposed approach: (\textbf{Left}) Localization error, (\textbf{Center}) CPU utilization per iteration by each robot and (\textbf{Right}) Average processing time for position estimation by each robot.}
    \label{fig:loc_plots}
\end{figure}
\fi

\subsubsection{Accuracy}
Compared to alternative techniques, HGP-RL exhibits a considerable improvement in the average RMSE. 
As shown in Fig.~\ref{fig:loc_performance_plot}, the reduced RMSE values obtained by HGP-RL prove its advantage in various settings and starting conditions. %Fig.~\ref{fig:loc_trjectory_plots} shows the Gaussian prediction for global and local frames along with the ground truth and predicted trajectories.
It is more accurate than TRN by roughly 42 percent and roughly 30.84\% more accurate than DGORL in the highest noise level, showing up to 2x improvements. 
%Its accuracy is roughly 34.01\% more than the Sparse version of HGP-RL without Hierarchy or a ratio of about 1.52. The Hierarchy version of HGP-RL has a little higher average RMSE than the Dense version of the algorithm without Hierarchy, though. 
In addition to having better accuracy than previous approaches, the HGP-RL technique also shows less fluctuation in RMSE values among the three robots. 
Our analysis of the RMSE values for various approaches highlights the greater accuracy of our HGP-RL strategy in predicting robots' relative positions. 
The reliability of this consistency for localization tasks is highlighted. The HGP-RL establishes itself as a good contender for practical robot localization tasks by offering higher accuracy and consistency.

\subsubsection{Efficiency}
Our proposed HGP-RL method greatly outperforms other approaches when compared to processing times and CPU use. HGP-RL uses 43.62\% of the CPU, which is significantly less than other methods, which use close to or more than 100\% of the CPU. Specifically, HGP-RL uses 45.05 \% less CPU than TRN, making it around 2.23 times more efficient. The HGP-RL shows a reduction of 64.67\% compared to DGORL's CPU consumption of 123.56\%, making it nearly three times as effective. 
%TODO We also looked at the efficiency of the HGP-RL without hierarchy by using the Sparse and Dense resolution GPR models. 
%The hierarchy version of HGP-RL, however, uses 23.18\% more CPU power than the Sparse version of the same algorithm without Hierarchy. However, it outperforms the Dense version of HGP-RL without Hierarchy in terms of efficiency, requiring 51.13\% less CPU and being nearly 2.05 times as efficient.

In terms of processing time, HGP-RL took around 33ms (a real-time performance at 30Hz frequency) and is more time-effective by being 23.59\% faster than TRN's processing time of 43.65 ms. HGP-RL is nearly twice as efficient as DGORL, which processes in 64.34 ms. 
%In contrast, the Sparse version of HGP-RL without Hierarchy exhibits a 55.69\% increase in processing time when compared to the HGP-RL with Hierarchy, showing that the Sparse version is more effective in this regard. However, the processing time efficiency of the HGP-RL with Hierarchy is 48.91\% greater than that of the Dense version of the HGP-RL without Hierarchy, or almost 1.96 times.
Furthermore, with HGP-RL, the robots share only two pieces of data: the AP location estimate and the current odometry position. This is significantly lower compared to the relative range of information shared in the other approaches. As we can see, the COM payload of HGP-RL is up to 70\% lower than the other approaches. 
As observed in Fig.~\ref{fig:loc_performance_plot}, the reduced CPU utilization, COM payload, and processing times for HGP-RL prove its efficiency for implementation in low-power computing devices and resource-constrained robots.

\subsubsection{Scalability}
We also have performed experiments to validate the scalability and robustness of the proposed HGP-RL approach. We have varied numbers of robots from 3 to 10 and conducted experiments in increasing sizes of simulation environments to allow sufficient space for robots to navigate. Table~\ref{tab:scalability_perforamnce} demonstrates the scalability of the HGP-RL system under constant resolution and hierarchy-level settings as other experiments to maintain high accuracy. The localization error (RMSE) modestly rises from 0.073 m to 0.124 m for 3 to 10 robots, indicating a slight decrease in accuracy with more complex scenarios involving more robots and larger dimensions. Processing time exhibits a more pronounced increase, growing from 33 milliseconds for the smallest dimension with three robots to 121 milliseconds for the largest dimension with ten robots (which takes more than 920 ms for a full-resolution GPR instead of HGP). This trend underscores the system's handling of increased complexity through both spatial expansion and a larger robot cohort, balancing accuracy against computational demands.

\begin{table}[h!]
%\vspace{-2mm}
\caption{Scalability results: HGP-RL performance for varying dimension and number of robots}
\label{tab:scalability_perforamnce}
%\vspace{-2mm}
\centering
\begin{tabular}{lccc}
\hline
Dimension & No. of Robots & RMSE (m) & Proc. Time (msec) \\ \hline
$3\times2$ & 3 & 0.073 & 33 \\
$4\times3$ & 6 & 0.092 & 59 \\
$5\times4$ & 8 & 0.103 & 97 \\
$6\times5$ & 10 & 0.124 & 121 \\ \hline
\end{tabular}
%\vspace{-2mm}
\end{table}

\subsubsection{Impact of the RSSI noise levels}
As we rely on the RSSI measurements to build the GPR map, the noise level in the RSSI measurements can affect the accuracy of the localization. 
To analyze this effect, we simulated different noise levels in the measured RSSI values. 
The overlay plot in Fig.~\ref{fig:loc_performance_plot} shows that the simulated RSSI at noise level 4 dBm represents real-world RSSI observations. 
Although DGORL performed better than TRN as observed in \cite{latif2022dgorl}, HGP-RL has demonstrated 2x higher accuracy than both approaches. Furthermore, the results have shown that the HGP-RL has lower RMSE (high accuracy) among all techniques, even under high noise levels. But, the improvement in accuracy is less pronounced with increasing noise levels.

%\subsection{Discussion}

%TODO Scalability explain number of robots - distributed and coomplexity of GP >> RL
%TODO - limitation wrt time sync or delayed data?

%Add \cite{miyagusuku2018data} for multiple access points (roaming scenario)
%Add \cite{dong2022mr} GMM for complex RSSI signal distributions
%Add \cite{miyagusuku2018data} GPR with path loss models known - known environment

\begin{figure}[t]
\includegraphics[width=0.96\linewidth]{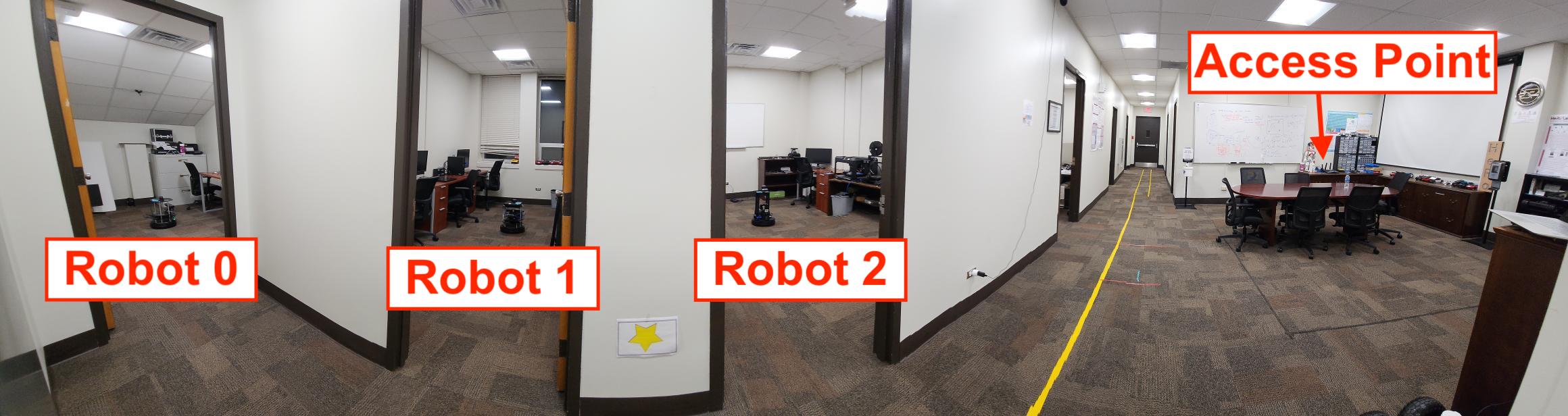}  
\vspace*{2mm}
\includegraphics[width=0.59\linewidth]{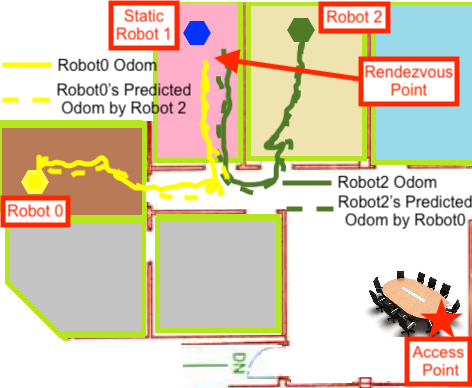}
\includegraphics[width=0.36\linewidth]{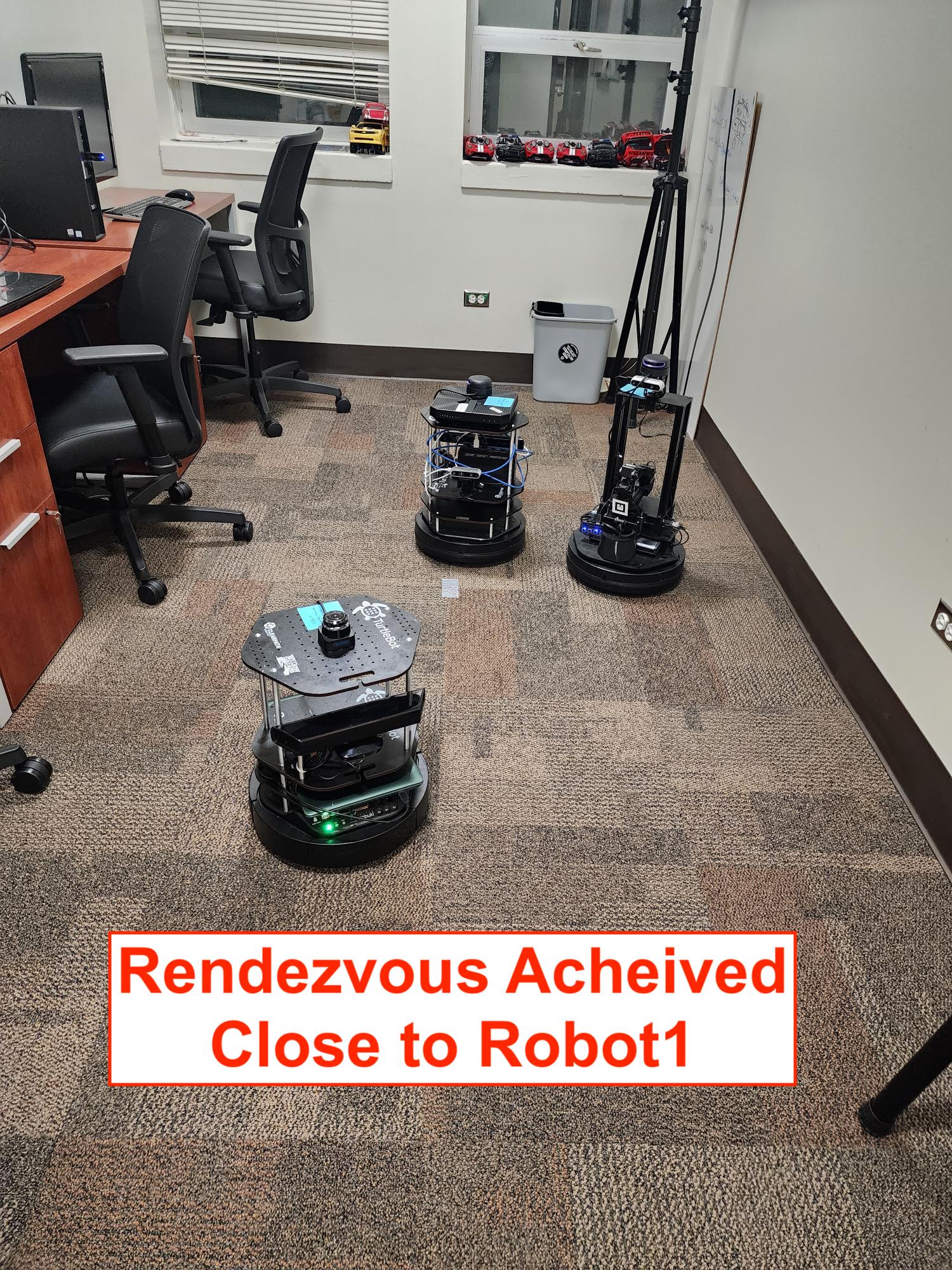}
%\end{subfigure}
%\vspace{-2mm}
\caption{A trial from the multi-robot rendezvous experiment: Initial (top) and final (bottom) state of the robots and their trajectories using HGP-RL.}
    \label{fig:rendezvous}
    %\vspace{-4mm}
\end{figure}

%\iffalse % test

\section{Real-world Experiments}
\label{sec:realworld}
To validate the practicality and generalizability, we implemented the HGP-RL approach in ROS-Neotic framework to perform a multi-robot rendezvous task using three Tutlebot2e robots in a large 10m x 13m multiroom lab environment. All robots are configured as fully connected to the same Wi-Fi AP. The robots are initially located in different rooms (without visibility to each other and the AP), and the RSSI map has non-line-of-sight conditions.
%To perform the rendezvous, robots then use the relative localization method discussed in Section~\ref{sec:relative_localization} of other robots. 
% Robots operate in a distributed way to find the direction to pursue using perceived relative locations of other robots in their local frame of reference using  
% \begin{equation}
%     _i \dot{\mathbf{p}}^i = \frac{1}{N} \sum_{j\in \mathcal{N}_i} (_i \mathbf{p}^j - {}_i \mathbf{p}^i) .
%     \label{eqn:rendezvous} 
% \end{equation}
% Here, $\mathcal{N}_i$ is the set of neighbor robots of the robot $i$. 
% The velocity (direction) obtained from Eq.~\eqref{eqn:rendezvous} is then given to the robot's mapping and autonomous motion planner for low-level path planning and obstacle avoidance (we used ROS gmapping + move\_base packages for this implementation). 
We performed five trials and successfully achieved a rendezvous of all robots within a small threshold distance between the robots. 

A sample of the initial and rendezvous positions and their trajectories can be seen in Fig.~\ref{fig:rendezvous}. The robots could locate the AP within $0.57 \pm 0.12$ m accuracy, and the relative trajectory error (between the odometry and predicted trajectory of other robots) was within $0.42 \pm 0.09m$ on average. The \textbf{attached experiment video} shows the performance of GRPL in real time.
%Note, we are able to perform relative localization in a large workspace bounded by the communication range, which is much smaller than a sensor range. 
The experiment validated the practicality of the approach in handling real-world scenarios with noisy RSSI values and occluded non-line-of-sight conditions. It also demonstrates the applicability of HGP-RL to most multi-robot operations, such as exploration and formation control.
Together, the experimental evidence suggests that HGP-RL is a scalable, practical, and reliable approach to perform multi-robot relative localization. HGP-RL can be readily applied to scenarios where no visibility between the robots or robots' trajectories never overlap (i.e., in scenarios where loop closure is not possible but is a requirement for map-based or feature-based localization approaches).

\section{Conclusion}
We proposed a Hierarchical Gaussian Processes-based Relative Localization (HGP-RL) approach for multi-robot systems. HGP-RL combines hierarchical inferencing over the RSSI map with a novel AP-oriented relative localization using the ubiquitous RSSI data from a single Wi-Fi AP. Our method addresses the limitations of existing solutions by offering high accuracy while reducing computational efficiency to enable accurate and efficient relative localization. Experiment results demonstrated that the HGP-RL approach outperforms state-of-the-art methods such as GPR-variants, TRN, and DGORL regarding accuracy, computational and communication efficiency. Moreover, the proposed method consistently performs across different experiments, making it a reliable and practical choice for localization tasks. 
%Future work will focus on overcoming some of the limitations by integrating reliable signal processing techniques, Angle of Arrival methods, and synergistic information fusion to further enhance the performance of the proposed HGP-RL.

\bibliographystyle{IEEEtran}
\bibliography{ref}

\end{document}